\documentclass{article}

\usepackage{times}
\usepackage{amssymb}
\usepackage{amsmath}
\usepackage{amsthm}
\usepackage{bm}
\usepackage{tikz}
\usepackage{tkz-euclide}
\usetkzobj{all}
\usepackage{chngcntr}
\usepackage{verbatim}
\usepackage[final,nonatbib]{nips_2016}
\usepackage{mathtools}
\usepackage{esvect}
\RequirePackage{algorithm}
\RequirePackage{algorithmic}
\usepackage{color}

\newcommand{\E}{{\mathbf{E}}} 
\renewcommand{\S}{{\cal S}} 
\newcommand{\T}{{\cal T}} 

\DeclareMathOperator*{\argmin}{\text{argmin}}

\newtheorem{theorem}{Theorem}
\newtheorem{lemma}[theorem]{Lemma}
\newtheorem{proposition}[theorem]{Proposition}

\newtheorem{assumption}{Assumption}
\graphicspath{ {./images/} }

\input{mysymbol.sty}
\def\whp{\text{w.h.p}}

\author{Aryan Mokhtari\\
University of Pennsylvania\\
Philadelphia, PA\\
aryanm@seas.upenn.edu
\and
\textbf{Alejandro Ribeiro}\\
University of Pennsylvania\\
Philadelphia, PA\\
aribeiro@seas.upenn.edu}

\renewcommand{\comment}[1]{}
\newcolumntype{S}{>{\centering\arraybackslash} m{.10\linewidth} }
\newcolumntype{T}{>{\centering\arraybackslash} m{.30\linewidth} }
\title{Adaptive Newton Method for Empirical Risk Minimization to Statistical Accuracy}
\begin{document}

\maketitle

\thispagestyle{empty}
\begin{abstract}
We consider empirical risk minimization for large-scale datasets. We introduce Ada Newton as an adaptive algorithm that uses Newton's method with adaptive sample sizes. The main idea of Ada Newton is to increase the size of the training set by a factor larger than one in a way that the minimization variable for the current training set is in the local neighborhood of the optimal argument of the next training set. This allows to exploit the quadratic convergence property of Newton's method and reach the statistical accuracy of each training set with only one iteration of Newton's method. We show theoretically and empirically that Ada Newton can double the size of the training set in each iteration to achieve the statistical accuracy of the full training set with about two passes over the dataset.
\end{abstract}


%
\section{Introduction}\label{sec_intro}

A hallmark of empirical risk minimization (ERM) on large datasets is that evaluating descent directions requires a complete pass over the dataset. Since this is undesirable due to the large number of training samples, stochastic optimization algorithms with descent directions estimated from a subset of samples are the method of choice. First order stochastic optimization has a long history \cite{robbins1951stochastic,polyak1992acceleration} but the last decade has seen fundamental progress in developing alternatives with faster convergence. A partial list of this consequential literature includes Nesterov acceleration \cite{nesterov2007gradient, beck2009fast}, stochastic averaging gradient \cite{ roux2012stochastic,defazio2014saga}, variance reduction \cite{johnson2013accelerating, xiao2014proximal}, dual coordinate methods \cite{ shalev2013stochastic, shalev2016accelerated}, and hybrid algorithms \cite{zhang2013linear, konevcny2013semi}.

When it comes to stochastic second order methods the first challenge is that while {\it evaluation} of Hessians is as costly as evaluation of gradients, the stochastic {\it estimation} of Hessians has proven more challenging. This difficulty is addressed by incremental computations in \cite{gurbuzbalaban2015globally} and subsampling in \cite{erdogdu2015convergence} or circumvented altogether in stochastic quasi-Newton methods \cite{schraudolph2007stochastic, bordes2009sgd, mokhtari2014res,JMLR:v16:mokhtari15a,moritz2015linearly, gower2016stochastic}. 
Despite this incipient progress it is nonetheless fair to say that the striking success in developing stochastic first order methods is not matched by equal success in the development of stochastic second order methods. This is because even if the problem of estimating a Hessian is solved there are still four challenges left in the implementation of Newton-like methods in ERM:

\begin{itemize}
\item[\bf(i)] Global convergence of Newton's method requires implementation of a line search subroutine and line searches in ERM require a complete pass over the dataset.
\item[\bf(ii)] The quadratic convergence advantage of Newton's method manifests close to the optimal solution but there is no point in solving ERM problems beyond their statistical accuracy.
\item[\bf(iii)] Newton's method works for strongly convex functions but loss functions are not strongly convex for many ERM problems of practical importance.
\item[\bf(iv)] Newton's method requires inversion of Hessians which is costly in large dimensional ERM.
\end{itemize}

Because they can't use line searches [cf. (i)], must work on problems that may be not strongly convex [cf. (iii)], and never operate very close to the optimal solution [cf (ii)], stochastic Newton-like methods never experience quadratic convergence. They do improve convergence constants in ill-conditioned problems but they still converge at linear rates.

In this paper we attempt to overcome (i)-(iv) with the Ada Newton algorithm that combines the use of Newton iterations with adaptive sample sizes \cite{daneshmand2016starting}. Say the total number of available samples is $N$, consider subsets of $n\leq N$ samples, and suppose the statistical accuracy of the ERM associated with $n$ samples is $V_n$ (Section \ref{sec_erm}). In Ada Newton we add a quadratic regularization term of order $V_n$ to the empirical risk -- so that the regularized risk also has statistical accuracy $V_n$ -- and assume that for a certain initial sample size $m_0$, the problem has been solved to its statistical accuracy $V_{m_0}$. The sample size is then increased by a factor $\alpha>1$ to $n=\alpha m_0$. We proceed to perform a single Newton iteration with unit stepsize and prove that the result of this update solves this extended ERM problem to its statistical accuracy (Section \ref{sec_ada_newton}). This permits a second increase of the sample size by a factor $\alpha$ and a second Newton iteration that is likewise guaranteed to solve the problem to its statistical accuracy. Overall, this permits minimizing the empirical risk in $\alpha/({\alpha-1})$ passes over the dataset and inverting $\log_\alpha N$ Hessians. Our theoretical and  numerical analyses indicate that we can make $\alpha=2$. In that cases we can optimize to within statistical accuracy in about 2 passes over the dataset and after inversion of about $3.32\log_{10} N$ Hessians. 


%
\section{Empirical risk minimization}\label{sec_erm}

We want to solve ERM problems to their statistical accuracy. To state this problem formally consider an argument $\bbw\in\reals^p$, a random variable $Z$ with realizations $z$ and a convex loss function $f(\bbw;z)$. We want to find an argument $\bbw^*$ that minimizes the statistical average loss $L(\bbw):=\E_Z[f(\bbw,Z)]$,
\begin{align}\label{eqn_stat_loss}
   \bbw^* := \argmin_{\bbw} L(\bbw) 
           = \argmin_{\bbw} \E_Z[f(\bbw,Z)].
\end{align}
The loss in \eqref{eqn_stat_loss} can't be evaluated because the distribution of $Z$ is unknown. We have, however, access to a training set $\ccalT = \{z_1,\ldots,z_N\}$ containing
$N$ independent samples $z_1,\ldots,z_N$ that we can use to estimate $L(\bbw)$. We therefore consider a subset $\ccalS_n \subseteq\ccalT$ and settle for minimization of the empirical risk $L_{n}(\bbw):=(1/n)\sum_{k=1}^{n} f (\bbw,z_k)$,
\begin{align}\label{eqn_empirical_loss}
   \bbw_{n}^{\dagger} := \argmin_{\bbw} L_{n}(\bbw) 
                 = \argmin_{\bbw} \frac{1}{n} \sum_{k=1}^{n} f (\bbw,z_k),
\end{align}
where, without loss of generality, we have assumed $\ccalS_n=\{z_1,\ldots,z_n\}$ contains the first $n$ elements of $\ccalT$. The difference between the empirical risk in \eqref{eqn_empirical_loss} and the statistical loss in \eqref{eqn_stat_loss} is a fundamental quantities in statistical learning. We assume here that there exists a constant $V_n$, which depends on the number of samples $n$, that upper bounds their difference for all $\bbw$ with high probability ($\whp$),
\begin{align}\label{eqn_loss_minus_erm}
   \sup_{\bbw}|L(\bbw) - L_{n}(\bbw) |  \leq V_n,  \qquad\whp.
\end{align}
That the statement in \eqref{eqn_loss_minus_erm} holds with $\whp$ means that there exists a constant $\delta$ such that the inequality holds with probability at least $1-\delta$. The constant $V_n$ depends on $\delta$ but we keep that dependency implicit to simplify notation. For subsequent discussions, observe that bounds $V_n$ of order $V_n = O(1/\sqrt{n})$ date back to the seminal work of Vapnik -- see e.g., \cite[Section 3.4]{{vapnik1998statistical}}. Bounds of order $V_n = O(1/n)$ have been derived more recently under stronger regularity conditions that are not uncommon in practice, \cite{bartlett2006convexity, frostig2014competing, bousquet2008tradeoffs}

An important consequence of \eqref{eqn_stat_loss} is that there is no point in solving \eqref{eqn_empirical_loss} to an accuracy higher than $V_n$. Indeed, if we find a variable $\bbw$ for which $L_{n}(\bbw_n)-L_{n}(\bbw^\dagger)\leq V_n$ finding a better approximation of $\bbw^\dagger$ is moot because \eqref{eqn_loss_minus_erm} implies that this is not necessarily a better approximation of the minimizer $\bbw^*$ of the statistical loss. We say the variable $\bbw_n$ solves the ERM problem in \eqref{eqn_empirical_loss} to within its statistical accuracy. In particular, this implies that adding a regularization of order $V_n$ to \eqref{eqn_empirical_loss} yields a problem that is essentially equivalent. We can then consider a quadratic regularizer of the form $c V_n/2 \|\bbw\|^2$ to define the regularized empirical risk $R_{n}(\bbw):= L_{n}(\bbw) + (c V_n/2)\|\bbw\|^2$ and the corresponding optimal argument
\begin{align}\label{eqn_empirical_loss_regularized}
   \bbw_{n}^{*} := \argmin_{\bbw} R_{n}(\bbw)
                       = \argmin_{\bbw} L_{n}(\bbw) + \frac{c V_n}{2} \|\bbw\|^2.
\end{align}
Since the regularization in \eqref{eqn_empirical_loss_regularized} is of order $V_n$ and \eqref{eqn_loss_minus_erm} holds, the difference between $R_{n}(\bbw_n^*)$ and $L(\bbw^*)$ is also of order $V_n$ -- this may be not as immediate as it seems; see \cite{shalev2010learnability}. Thus, we can say that a variable $\bbw_n$ satisfying $R_{n}(\bbw_n) - R_{n}(\bbw_n^*) \leq V_n$ solves the ERM problem to within its statistical accuracy. We accomplish this in this paper with the Ada Newton algorithm.


%
\section{Ada Newton}\label{sec_ada_newton}

To solve \eqref{eqn_empirical_loss_regularized} suppose the problem has been solved to within its statistical accuracy for a set $\ccalS_m\subset\ccalS_n$ with $m = n/\alpha$ samples. Therefore, we have found a variable $\bbw_m$ for which $R_{m}(\bbw_m) - R_{m}(\bbw_m^*) \leq V_m$. We want to update $\bbw_m$ to obtain a variable $\bbw_n$ that estimates $\bbw_n^*$ with accuracy $V_n$. To do so compute the gradient of the risk $R_n$ evaluated at $\bbw_m$
\begin{align}\label{eqn_gradient}
    \nabla R_n (\bbw_m) 
           = \frac{1}{n}\sum_{k=1}^{n} \nabla f (\bbw_m,z_k) + cV_n \bbw_m,
\end{align}
as well as the Hessian $\bbH_n$ of $R_n$ evaluated at $\bbw_m$
\begin{align}\label{eqn_hessian}
   \bbH_n := \nabla^2 R_n (\bbw_m) 
           = \frac{1}{n}\sum_{k=1}^{n} \nabla^2 f (\bbw_m,z_k) + cV_n\bbI,
\end{align}
and update $\bbw_m$ with the Newton step of the regularized risk $R_n$ to compute
\begin{align}\label{eqn_ada_mewton}
   \bbw_n = \bbw_m - \bbH_n^{-1} \nabla R_n (\bbw_m).
\end{align}

The main contribution of this paper is to derive a condition that guarantees that $\bbw_n$ solves $R_n$ to within its statistical accuracy $V_n$.

%
\begin{theorem}\label{the_main_result_theorem}
Consider the variable $\bbw_m$ as a $V_m$-optimal solution of the risk $R_{m}$, i.e., a solution such that $R_{m}(\bbw_m)- R_{m}(\bbw_m^*) \leq V_m$. Let $n=\alpha m > m$, consider the risk $R_{n}$ associated with sample set $\S_n\supset\S_m$, and suppose assumptions \ref{convexity_lip_assumption} - \ref{grad_cond} hold. If the sample size $n$ is chosen such that
\begin{equation}\label{cond_1}
   \left(\frac{2(M+cV_m)V_{m}}{cV_n}\right)^{1/2}\!
                     + \frac{2(n-m)}{nc^{1/2}}     
                     + \frac{\left((2+\sqrt{2})c^{1/2} +c\|\bbw^*\|\right)(V_m-V_n)}{(cV_n)^{1/2}}\leq \frac{1}{4}
\end{equation}
and 
\begin{equation}\label{cond_2}
144 \left(  V_m +  \frac{2(n-m)}{n} \left(V_{n-m}+V_m\right)
 +2\left(V_m-V_n\right)
 + \frac{c(V_m-V_n)}{2}\|\bbw^*\|^2  \right)^2
\leq V_n
\end{equation}
are satisfied, then the variable $\bbw_n$, which is the outcome of applying one Newton step on the variable $\bbw_m$ as in \eqref{eqn_ada_mewton}, has sub-optimality error $V_n$ with high probability, i.e., 
\begin{equation}\label{imp_result}
 R_{n}(\bbw_n)- R_{n}(\bbw_n^*) \leq V_n, \qquad \whp.
\end{equation} \end{theorem}

%
\begin{proof} See section \ref{sec_convergence}.\end{proof}

%
Theorem \ref{the_main_result_theorem} states conditions under which we can iteratively increase the sample size while applying single Newton iterations without line search and staying within the statistical accuracy of the regularized empirical risk. The constants in \eqref{cond_1} and \eqref{cond_2} are not easy to parse but we can understand them qualitatively if we focus on large $m$. This results in a simpler condition that we state next.

%
{\linespread{1.3}
\begin{algorithm}[t] \begin{algorithmic}[1]
\STATE \textbf{Parameters:} Sample size increase constants $\alpha_0>1$ and $0<\beta<1$.
\STATE \textbf{Input:} Initial sample size $n=m_0$ and 
                       argument $\bbw_{n} = \bbw_{m_0}$ with 
                       $\| \nabla R_{n}(\bbw_n)\| < (\sqrt{2 c}) V_n$
\WHILE [main loop]{$n\leq N$} 
   \STATE Update argument and index:\ $\bbw_m=\bbw_n$ and $m=n$. 
          Reset factor $\alpha=\alpha_0$ .     
   \REPEAT  [sample size backtracking loop] 
      \STATE Increase sample size:\ $n=\max\{\alpha m, N\}$. 
      \STATE Compute gradient [cf. \eqref{eqn_gradient}]: \ 
             ${
                 \nabla R_n (\bbw_m) = (1/n)\sum_{k=1}^{n}\nabla f(\bbw_m,z_k) + cV_n \bbw_m
             }$
      \STATE Compute Hessian [cf. \eqref{eqn_hessian}]: \
             ${
                 \bbH_n = (1/n)\sum_{k=1}^{n} \nabla^2 f (\bbw_m,z_k)+ cV_n \bbI
             }$
      \STATE Newton Update [cf. \eqref{eqn_ada_mewton}]: \
             ${
                 \bbw_n = \bbw_m - \bbH_n^{-1} \nabla R_n (\bbw_m)
             }$
      \STATE Compute gradient [cf. \eqref{eqn_gradient}]: \ 
             ${
                  \nabla R_{n}(\bbw_n)=(1/n)\sum_{k=1}^{n} \nabla f (\bbw_n,z_k) +cV_n\bbw_n
             }$
      \STATE Backtrack sample size increase $\alpha=\beta\alpha$.       
   \UNTIL {$\| \nabla R_{n}(\bbw_n)\| > (\sqrt{2 c}) V_n$} 
\ENDWHILE
\end{algorithmic}
\caption{{Ada Newton}}\label{alg:AdaNewton} \end{algorithm}}

%
\begin{proposition}\label{prop_condition}
Consider a learning problem in which the statistical accuracy satisfies $V_m\leq\alpha V_n$ for $n=\alpha m$ and $\lim_{n\to\infty} V_n =0$. If the regularization constant $c$ is chosen so that
\begin{equation}\label{cond_1_simple}
   \left(\frac{2\alpha M}{c}\right)^{1/2}
                        +   \frac{2\alpha}{(\alpha-1)c^{1/2}}  < \frac{1}{4},
\end{equation}
then, there exists a sample size $\tdm$ such that \eqref{cond_1} and \eqref{cond_2} are satisfied for all $m>\tdm$ and $n=\alpha m$. In particular, if $\alpha=2$ we can satisfy \eqref{cond_1} and \eqref{cond_2} with $c> 64(\sqrt{M}+2)^2$.
\end{proposition}
%
\begin{proof} That the condition in \eqref{cond_2} is satisfied for all $m>\tdm$ follows simply because the left hand side is of order $V_m^2$ and the right hand side is of order $V_n$. To show that the condition in \eqref{cond_1} is satisfied for sufficiently large $m$ observe that the third summand in \eqref{cond_1} is of order $O((V_m - V_n)/V_n^{1/2})$ and vanishes for large $m$. In the second summand of \eqref{cond_1} we make $n=\alpha m$ to obtain the second summand in \eqref{cond_1_simple} and in the first summand replace the ratio $V_m/V_n$ by its bound $\alpha$ to obtain the first summand of \eqref{cond_1_simple}. To conclude the proof just observe that the inequality in \eqref{cond_1_simple} is strict. \end{proof}

%
The condition $V_m\leq\alpha V_n$ is satisfied if $V_n=1/n$ and is also satisfied if $V_n=1/\sqrt{n}$ because $\sqrt{\alpha}<\alpha$. This means that for most ERM problems we can progress geometrically over the sample size and arrive at a solution $\bbw_N$ that solves the ERM problem $R_N$ to its statistical accuracy $V_N$ as long as \eqref{cond_1_simple} is satisfied . 

The result in Theorem \ref{the_main_result_theorem} motivates definition of the Ada Newton algorithm that we summarize in Algorithm \ref{alg:AdaNewton}. The core of the algorithm is in steps 6-9. Step 6 implements an increase in the sample size by a factor $\alpha$ and steps 7-9 implement the Newton iteration in \eqref{eqn_gradient}-\eqref{eqn_ada_mewton}.  The required input to the algorithm is an initial sample size $m_0$ and a variable $\bbw_{m_0}$ that is known to solve the ERM problem with accuracy $V_{m_0}$. Observe that this initial iterate doesn't have to be computed with Newton iterations. The initial problem to be solved contains a moderate number of samples $m_0$, a mild condition number because it is regularized with constant $cV_{m_0}$, and is to be solved to a moderate accuracy $V_{m_0}$ -- recall that $V_{m_0}$ is of order $V_{m_0}=O(1/m_0)$ or order $V_{m_0}=O(1/\sqrt{{m_0}})$ depending on regularity assumptions. Stochastic first order methods excel at solving problems with moderate number of samples $m_0$ and moderate condition to moderate accuracy.

We remark that the conditions in Theorem \ref{the_main_result_theorem} and Proposition \ref{prop_condition} are conceptual but that the constants involved are unknown in practice. In particular, this means that the allowed values of the factor $\alpha$ that controls the growth of the sample size are unknown a priori. We solve this problem in Algorithm~\ref{alg:AdaNewton} by backtracking the increase in the sample size until we guarantee that $\bbw_n$ minimizes the empirical risk $R_{n}(\bbw_n)$ to within its statistical accuracy. This backtracking of the sample size is implemented in Step 11 and the optimality condition of $\bbw_n$ is checked in Step 12. The condition in Step 12 is on the gradient norm that, because $R_n$ is strongly convex, can be used to bound the suboptimality $R_{n}(\bbw_n)- R_{n}(\bbw_n^*)$ as
\begin{equation}\label{optimality_check}
   R_{n}(\bbw_n)- R_{n}(\bbw_n^*) \leq \frac{1}{2cV_n} \|\nabla R_{n}(\bbw_n)\|^2.
\end{equation}
Observe that checking this condition requires an extra gradient computation undertaken in Step 10. That computation can be reused in the computation of the gradient in Step 5 once we exit the backtracking loop. We emphasize that when the condition in \eqref{cond_1_simple} is satisfied, there exist a sufficiently large $m$ for which the conditions in Theorem \ref{the_main_result_theorem} are satisfied for $n=\alpha m$. This means that the backtracking condition in Step 12 is satisfied after one iteration and that, eventually, Ada Newton progresses by increasing the sample size by a factor $\alpha$. This means that Algorithm \ref{alg:AdaNewton} can be thought of as having a damped phase where the sample size increases by a factor smaller than $\rho$ and a geometric phase where the sample size grows by a factor $\rho$ in all subsequent iterations. The computational cost of this geometric phase is of not more than $\alpha/({\alpha-1})$ passes over the dataset and requires inverting not more than $\log_\alpha N$ Hessians. If $c> 64(\sqrt{M}+2)^2$, we make $\alpha=2$ for optimizing to within statistical accuracy in about 2 passes over the dataset and after inversion of about $3.32\log_{10} N$ Hessians.


\section{Convergence Analysis}\label{sec_convergence}

In this section we study the proof of Theorem \ref{the_main_result_theorem}. In proving this result we first assume the following conditions are satisfied. 

\begin{assumption}\label{convexity_lip_assumption}
The loss functions $f(\bbw,\bbz)$ are convex with respect to $\bbw$ for all values of $\bbz$. Moreover, their gradients $\nabla f(\bbw,\bbz)$ are Lipschitz continuous with constant $M$
\begin{equation}\label{lip_ass_cond}
\|\nabla f(\bbw,\bbz)-\nabla f(\bbw',\bbz)\| \leq M \|\bbw-\bbw'\|,\qquad \text{for all $\bbz$.}
\end{equation}
\end{assumption}

\begin{assumption}\label{self_concor}
The loss functions $f(\bbw,\bbz)$ are self-concordant with respect to $\bbw$ for all $\bbz$. 
\end{assumption}


\begin{assumption}\label{grad_cond}
The difference between the gradients of the empirical loss $L_n$ and the statistical average loss $L$ is bounded by $V_n^{1/2}$ for all $\bbw$ with high probability,
\begin{align}\label{eqn_loss_minus_erm_2}
   \sup_{\bbw}\|\nabla L(\bbw) - \nabla L_{n}(\bbw) \|  \leq V_n^{1/2},  \qquad\whp.
\end{align}
\end{assumption}

The conditions in Assumption \ref{convexity_lip_assumption} imply that the average loss $L(\bbw)$ and the empirical loss $L_n(\bbw)$ are convex and their gradients are Lipschitz continuous with constant $M$. Thus, the empirical risk $R_n(\bbw)$ is strongly convex with constant $cV_n$ and its gradients $\nabla R_n(\bbw)$ are Lipschitz continuous with parameter $M+cV_n$. Likewise, the condition in Assumption \ref{self_concor} implies that the average loss $L(\bbw)$, the empirical loss $L_n(\bbw)$, and the empirical risk $R_n(\bbw)$ are also self-concordant. The condition in Assumption \ref{grad_cond} says that the gradients of the empirical risk converge to their statistical average at a rate of order $V_n^{1/2}$. If the constant $V_n$ in condition \eqref{eqn_loss_minus_erm} is of order not faster than $O(1/n)$ the condition in Assumption \ref{grad_cond} holds if the gradients converge to their statistical average at a rate of order $V_n^{1/2}=O(1/\sqrt{n})$. This is a conservative rate for the law of large numbers.


The main idea of the Ada Newton algorithm is introducing a policy for increasing the size of training set from $m$ to $n$ in a way that the current variable $\bbw_m$ is in the Newton quadratic convergence phase for the next regularized empirical risk $R_n$. In the following proposition, we characterize the required condition to guarantee staying in the local neighborhood of Newton's method.

\begin{proposition}\label{main_prop}
Consider the sets $\S_m$ and $\S_n$ as subsets of the training set $\T$ such that $\S_m \subset \S_n\subset \T$. We assume that the number of samples in the sets  $\S_m$ and $\S_n$ are $m$ and $n$, respectively. Further, define $\bbw_{m}$ as an $V_m$ optimal solution of the risk $R_{m}$, i.e., ${R_{m}(\bbw_m)-R_{m}(\bbw_m^*) \leq V_m}$. In addition, define $\lambda_n (\bbw):=\left(\nabla R_{n}(\bbw) ^T \nabla^2 R_{n}(\bbw) ^{-1} \nabla R_{n}(\bbw) \right)^{1/2}$ as the Newton decrement of variable $\bbw$ associated with the risk $R_{n}$. If Assumption \ref{convexity_lip_assumption}-\ref{grad_cond} hold, then Newton's method at point $\bbw_m$ is in the quadratic convergence phase for the objective function $R_{n}$, i.e., $\lambda_n(\bbw_m)<1/4$, if we have
\begin{equation}\label{prop_result}
   \left(\frac{2(M+cV_m)V_{m}}{cV_n}\right)^{1/2}\!
                            + \frac{{(2(n-m)}/{n}) V_n^{1/2} +(\sqrt{2c}+2\sqrt{c}+c\|\bbw^*\|)(V_m-V_n)}{(cV_n)^{1/2}}\leq \frac{1}{4}
\quad \text{w.h.p.}
\end{equation}
\end{proposition}

\begin{proof}
See Section 7.1 in the Appendix.
\end{proof}

From the analysis of Newton's method we know that if the Newton decrement $\lambda_n(\bbw)$ is smaller than $1/4$, the variable $\bbw$ is in the local neighborhood of Newton's method; see e.g., Chapter 9 of \cite{boyd04}. From the result in Proposition \ref{main_prop}, we obtain a sufficient condition to guarantee that $\lambda_n(\bbw_m)<1/4$ which implies that $\bbw_m$, which is a $V_m$ optimal solution for the regularized empirical loss $R_m$, i.e., $R_m(\bbw_m)-R_m(\bbw_m^*)\leq V_m$, is in the local neighborhood of the optimal argument of $R_n$ that Newton's method converges quadratically. 

Unfortunately, the quadratic convergence of Newton's method for self-concordant functions is in terms of the Newton decrement $\lambda_n(\bbw)$ and it does not necessary guarantee quadratic convergence in terms of objective function error. To be more precise, we can show that $\lambda_n (\bbw_n) \leq \gamma \lambda_n (\bbw_m)^2$; however, we can not conclude that the quadratic convergence of Newton's method implies $R_n(\bbw_n)-R_n(\bbw_n^*) \leq \gamma' (R_n(\bbw_m)-R_n(\bbw_n^*) )^2$. In the following proposition we try to characterize an upper bound for the error $R_n(\bbw_n)-R_n(\bbw_n^*)$ in terms of the squared error $(R_n(\bbw_m)-R_n(\bbw_n^*) )^2$ using the quadratic convergence property of Newton decrement.

\begin{proposition}\label{quadratic_convg_prop}
Consider $\bbw_{m}$ as a variable that is in the local neighborhood of the optimal argument of the risk $R_{n}$ where Newton's method has a quadratic convergence rate, i.e., $\lambda_n (\bbw_m)\leq1/4$. Recall the definition of the variable $\bbw_n$ in \eqref{eqn_ada_mewton} as the updated variable using Newton step. If Assumption \ref{convexity_lip_assumption} and \ref{self_concor} hold, then the difference $R_n(\bbw_n)-R_n(\bbw_n^*)$ is upper bounded by
\begin{equation}\label{prop_result_2}
		R_n(\bbw_n)-R_n(\bbw_n^*) \leq 144 (R_n(\bbw_m)-R_n(\bbw_n^*) )^2.
\end{equation}
\end{proposition}

	\begin{proof}
To prove the result in \eqref{prop_result_2} first we need to find upper and lower bounds for the difference $R_n(\bbw)-R_n(\bbw_n^*)$ in terms of the Newton decrement parameter $\lambda_n (\bbw)$. To do so, we use the result in Theorem 4.1.11 of \cite{nesterov1998introductory} which shows that
 \begin{equation}\label{proof_of_prop_10}
		  \lambda_n (\bbw) - \ln\left(1+\lambda_n (\bbw)\right) \leq  
	 R_n(\bbw)-R_n(\bbw_n^*)  \leq
	  - \lambda_n (\bbw) - \ln\left(1-\lambda_n (\bbw)\right).
\end{equation}
Note that we assume that $0<\lambda_n (\bbw)<1/4$. Thus, we can use the Taylor's expansion of $\ln(1+a)$ for $a=\lambda_n (\bbw)$ to show that $\lambda_n (\bbw) - \ln\left(1+\lambda_n (\bbw)\right)$ is bounded below by $(1/2)\lambda_n (\bbw)^2-(1/3)\lambda_n (\bbw)^3$. Since $0<\lambda_n (\bbw)<1/4$ we can show that $(1/6)\lambda_n (\bbw)^2\leq (1/2)\lambda_n (\bbw)^2-(1/3)\lambda_n (\bbw)^3$. Thus, the term $\lambda_n (\bbw) - \ln\left(1+\lambda_n (\bbw)\right)$ is bounded below by $(1/6)\lambda^2$. Likewise, we use Taylor's expansion of $\ln(1-a)$ for  $a=\lambda_n (\bbw)$ to show that $- \lambda_n (\bbw) - \ln\left(1-\lambda_n (\bbw)\right) $ is bounded above by $\lambda_n (\bbw)^2$ for $\lambda_n (\bbw)<1/4$; see e.g., Chapter 9 of \cite{boyd04}. Considering these bounds and the inequalities in \eqref{proof_of_prop_10} we can write 
 \begin{equation}\label{proof_of_prop_20}
	\frac{1}{6}\lambda_n (\bbw)^2\leq
	 R_n(\bbw)-R_n(\bbw_n^*)  \leq
	 \lambda_n (\bbw)^2.
\end{equation}

Recall that the variable $\bbw_m$ satisfies the condition $\lambda_n (\bbw_m)\leq 1/4$. Thus, according to the quadratic convergence rate of Newton's method for self-concordant functions \cite{boyd04}, we know that the Newton decrement has a quadratic convergence and we can write 
 \begin{equation}\label{proof_of_prop_30}
\lambda_n (\bbw_n) \leq 2 \lambda_n (\bbw_m)^2. 
\end{equation}
We use the result in \eqref{proof_of_prop_20} and \eqref{proof_of_prop_30} to show that the optimality error $R_n(\bbw_n)-R_n(\bbw_n^*)$ has an upper bound which is proportional to $(R_n(w_m)-R_n(\bbw_n^*) )^2$. In particular, we can write $ R_n(\bbw_n)-R_n(\bbw_n^*) \leq \lambda_n (\bbw_n)^2 $ based on the second inequality in \eqref{proof_of_prop_20}. This observation in conjunction with the result in \eqref{proof_of_prop_30} implies that
 \begin{align}\label{proof_of_prop_40}
 R_n(\bbw_n)-R_n(\bbw_n^*) \leq 4 \lambda_n (\bbw_m)^4.	
\end{align}
The first inequality in \eqref{proof_of_prop_20} implies that $\lambda_n (\bbw_m)^4\leq 36 (R_n(\bbw_m)-R_n(\bbw_n^*))^2 $. Thus, we can substitute $\lambda_n (\bbw_m)^4$ in \eqref{proof_of_prop_40} by $36 (R_n(\bbw_m)-R_n(\bbw_n^*))^2 $ to obtain the result in \eqref{prop_result_2}.
\end{proof}

The result in Proposition \ref{quadratic_convg_prop} provides an upper bound for the sub-optimality $R_n(\bbw_n)-R_n(\bbw_n^*)$ in terms of the sub-optimality of variable $\bbw_m$ for the risk $R_n$, i.e., $R_n(\bbw_m)-R_n(\bbw_n^*) $. Recall that we know that $\bbw_m$ is in the statistical accuracy of $R_m$, i.e., $R_m(\bbw_m)-R_m(\bbw_m^*)\leq V_m$, and we aim to show that the updated variable $\bbw_n$ stays in the statistical accuracy of $R_n$, i.e., $R_n(\bbw_n)-R_n(\bbw_n^*)\leq V_n$. This can be done by showing that the upper bound for $R_n(\bbw_n)-R_n(\bbw_n^*)$ in \eqref{prop_result_2} is smaller than $V_n$. We proceed to derive an upper bound for  the sub-optimality $ R_{n}(\bbw_m) - R_n(\bbw_n^*)$ in the following proposition. 

\begin{proposition}\label{main_theorem_444}
Consider the sets $\S_m$ and $\S_n$ as subsets of the training set $\T$ such that $\S_m \subset \S_n\subset \T$. We assume that the number of samples in the sets  $\S_m$ and $\S_n$ are $m$ and $n$, respectively. Further, define $\bbw_{m}$ as an $V_m$ optimal solution of the risk $R_{m}$, i.e., ${R_{m}(\bbw_m)-R_{m}^{*} \leq V_m}$. If Assumption \ref{convexity_lip_assumption}-\ref{grad_cond} hold, then the empirical risk error $ R_{n}(\bbw_m) - R_n(\bbw_n^*)$ of the variable $\bbw_m$ corresponding to the set $\S_n$ is bounded above by 
\begin{equation}\label{theorem_result_444}
 R_{n}(\bbw_m) - R_n(\bbw_n^*) \leq V_m +  \frac{2(n-m)}{n} \left(V_{n-m}+V_m\right)
 +2\left(V_m-V_n\right)
 + \frac{c(V_m-V_n)}{2}\|\bbw^*\|^2
\quad \text{w.h.p.}
\end{equation}
\end{proposition}

\begin{proof}
See Section 7.2 in the Appendix.
\end{proof}

The result in Proposition \ref{main_theorem_444} characterizes the sub-optimality of the variable $\bbw_m$, which is an $V_m$ sub-optimal solution for the risk $R_m$, with respect to the empirical risk $R_n$ associated with the set $\S_n$.

The results in Proposition \ref{main_prop}, \ref{quadratic_convg_prop}, and  \ref{main_theorem_444} lead to the result in Theorem \ref{the_main_result_theorem}. To be more precise, from the result in Proposition \ref{main_prop} we obtain that the condition in \eqref{cond_1} implies that $\bbw_m$ is in the local neighborhood of the optimal argument of $R_n$ and $\lambda_n(\bbw_m)\leq 1/4$. Hence, the hypothesis of Proposition \ref{quadratic_convg_prop} is satisfied and we have $R_n(\bbw_n)-R_n(\bbw_n^*) \leq 144 (R_n(\bbw_m)-R_n(\bbw_n^*) )^2$. This result paired with the result in Proposition \ref{main_theorem_444} shows that if the condition in \eqref{cond_2} is satisfied we can conclude that $R_n(\bbw_n)-R_n(\bbw_n^*)\leq V_n$ which completes the proof of Theorem \ref{the_main_result_theorem}.


\section{Experiments}
In this section, we study the performance of the proposed Ada Newton method and compare it with state-of-the-art in solving a large-scale classification problem. We use the protein homology dataset provided on KDD cup 2004 website. The dataset contains $N=145751$ samples and the dimension of each sample is $p=74$. We consider three algorithms to compare with the proposed Ada Newton method. One of them is the classic Newton's method with backtracking line search. The second algorithm is Stochastic Gradient Descent (SGD) and the last one is the SAGA algorithm introduced in \cite{defazio2014saga}. In our experiments, we use logistic loss and set the regularization parameters as $c=200$ and $V_n=1/n$.

 \begin{figure}[t]
	\centering
          \includegraphics[width=0.7\linewidth]{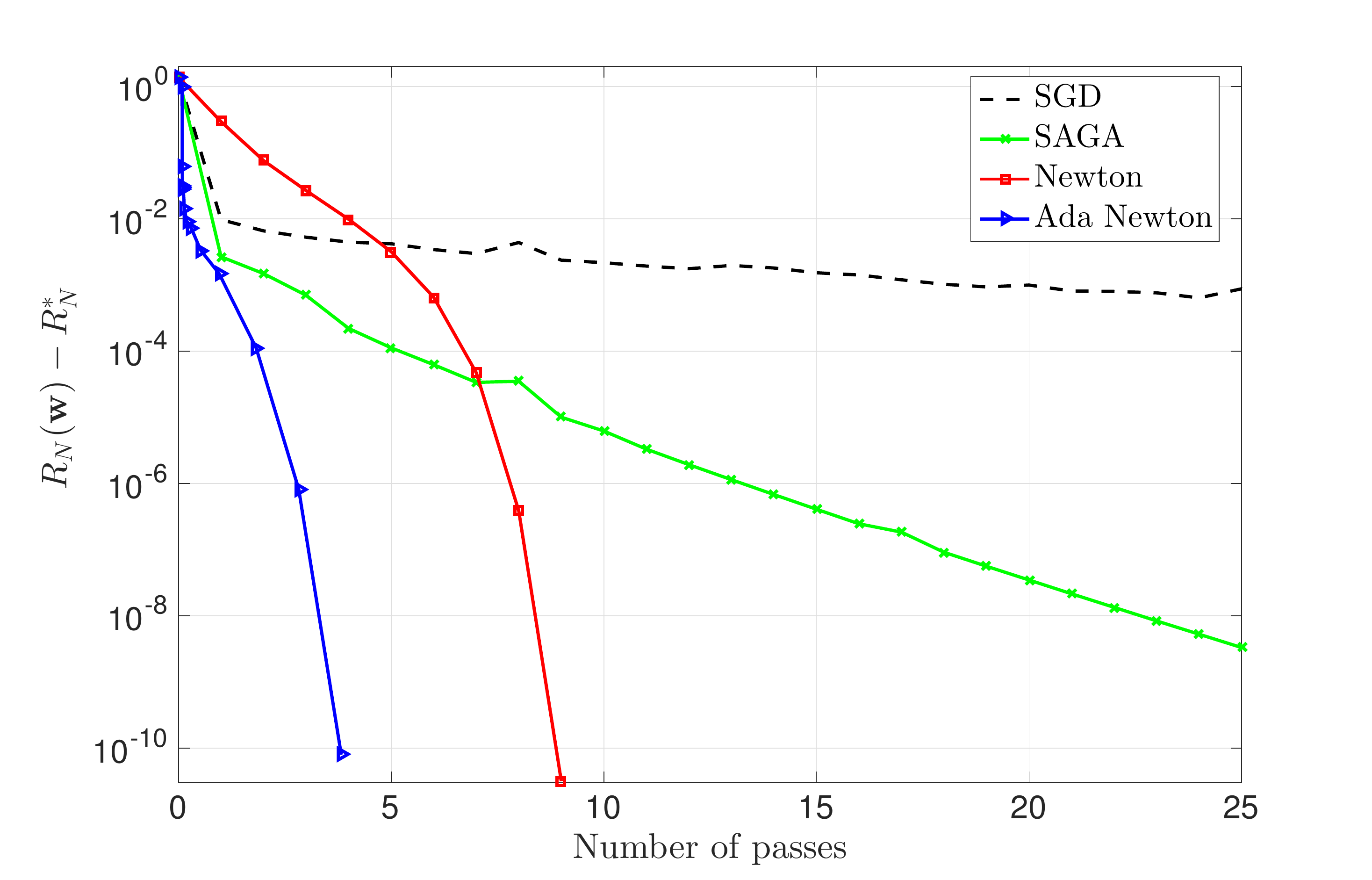} 
          \caption{Comparison of SGD, SAGA, Newton, and Ada Newton in terms of number of effective passes over dataset for the protein homology dataset.}
          \label{fig1}
\end{figure}

The stepsize of SGD in our experiments is $2\times 10^{-2}$. Note that picking larger stepsize leads to faster but less accurate convergence and choosing smaller stepsize improves the accuracy convergence with the price of slower convergence rate. The stepsize for SAGA is hand-optimized and the best performance has been observed for $\alpha=0.2$ which is the one that we use in the experiments. For  Newton's method, the backtracking line search parameters are $\alpha=0.4$ and $\beta=0.5$. In the implementation of Ada Newton we increase the size of the training set by factor $2$ at each iteration, i.e., $\alpha=2$ and we observe that the condition $\| \nabla R_{n}(\bbw_n)\| > (\sqrt{2 c}) V_n$ is always satisfied and there is no need for reducing the factor $\alpha$. Moreover, the size of initial training set is $m_0=124$. For the warmup step that we need to get into to the quadratic neighborhood of Newton's method we use the gradient descent method. In particular, we run gradient descent with stepsize $10^{-3}$ for $100$ iterations. Note that since the number of samples is very small at the beginning, $m_0=124$, and the regularizer is very large, the condition number of problem is very small. Thus, gradient descent is able to converge to a good neighborhood of the optimal solution in a reasonable time. Notice that the computation of this warm up process is very low and is equal to $12400$ gradient evaluations. This number of samples is less than $10\%$ of the full training set. In other words, the cost is less than $10\%$ of one pass over the dataset. Although, this cost is negligible, we consider it in comparison with SGD, SAGA, and Newton's method. We would like to mention that other algorithms such as Newton's method and stochastic algorithms can also be used for the warm up process; however, the gradient descent method sounds the best option since the gradient evaluation is not costly and the problem is well-conditioned for a small training set .

 \begin{figure}[t]
	\centering
          \includegraphics[width=0.7\linewidth]{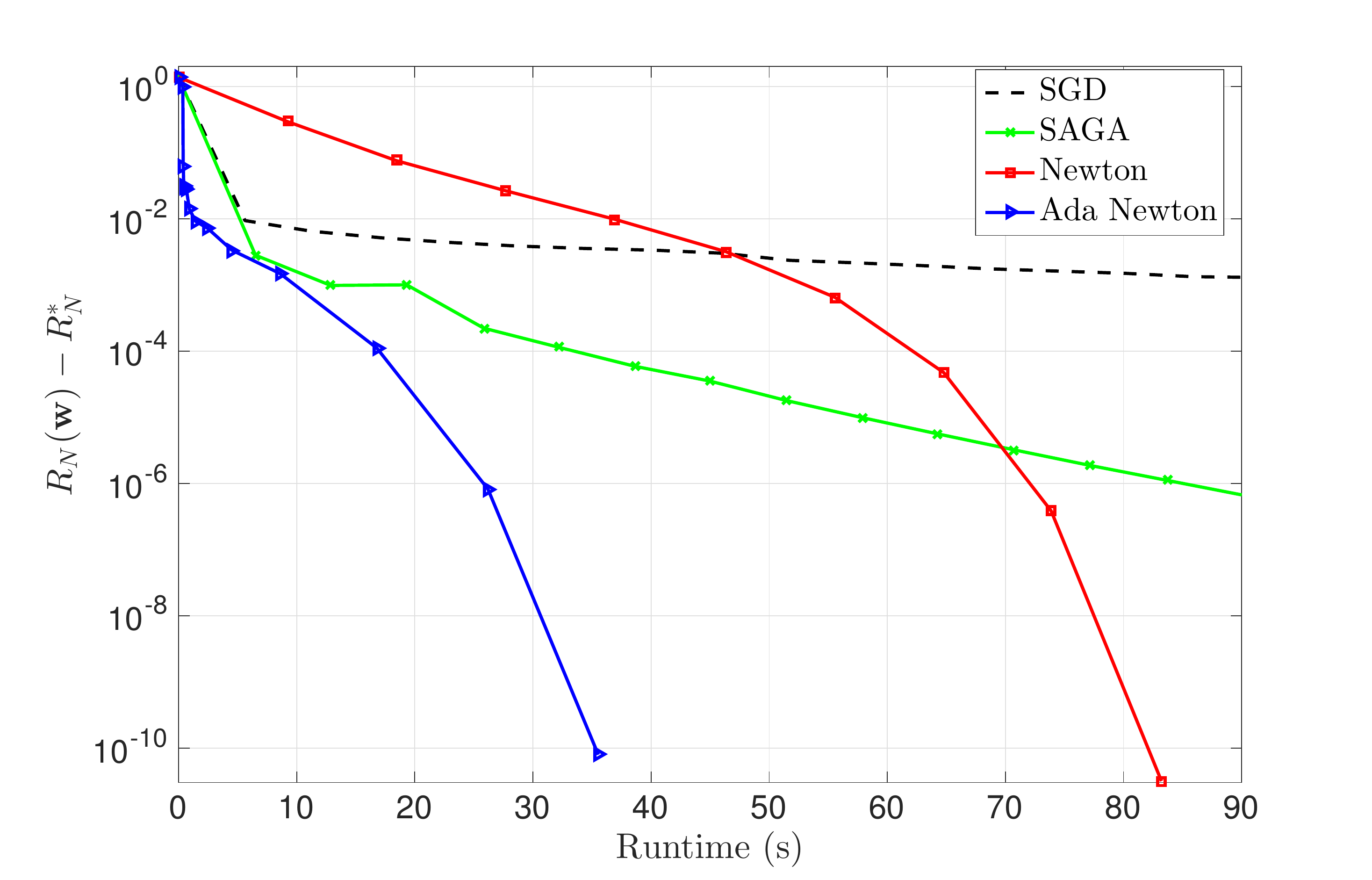} 
          \caption{Comparison of SGD, SAGA, Newton, and Ada Newton in terms of runtime for the protein homology dataset.}
          \label{fig2}
\end{figure}

Figure \ref{fig1} illustrates the convergence path of SGD, SAGA, Newton, and Ada Newton for the protein homology dataset. Note that the $x$ axis is the total number of samples used divided by the size of the training set $N=145751$ which we call number of passes over the dataset. As we observe, The best performance among the four algorithms belongs to Ada Newton. In particular, Ada Newton is able to achieve the accuracy of $R_N(\bbw)-R_n^*<1/N$ by $2.4$ passes over the dataset which is very close to theoretical result in Theorem 1 that guarantees accuracy of order $O(1/N)$ after $\alpha/(\alpha-1)=2$ passes over the dataset. To achieve the same accuracy of $1/N$ Newton's method requires $7.5$ passes over the dataset, while SAGA needs $10$ passes. The SGD algorithm can not achieve the statistical accuracy of order $O(1/N)$ even after $25$ passes over the dataset.

Although, Ada Newton and Newton outperform SAGA and SGD, their computational complexity are different. We address this concern by comparing the algorithms in terms of runtime. Figure \ref{fig2} demonstrates the convergence paths of the considered methods in terms of runtime. As we observe, Newton's method requires more time to achieve the statistical accuracy of $1/N$ relative to SAGA. This observation justifies the belief that Newton's method is not practical for large-scale optimization problems, since by enlarging $p$ or making the initial solution worse the performance of Newton's method will be even worse than the ones in Figure \ref{fig1}. Ada Newton resolves this issue by starting from small sample size which is computationally less costly. Ada Newton also requires Hessian inverse evaluations, but the number of inversions is proportional to $\log_\alpha N$. Moreover, the performance of Ada Newton doesn't depend on the initial point and the warm up process is not costly as we described before. We observe that Ada Newton outperforms SAGA significantly. In particular it achieves the statistical accuracy of $1/N$ in less than $25$ seconds, while SAGA achieves the same accuracy in $62$ seconds. Note that since the variable $\bbw_N$ is in the quadratic neighborhood of Newton's method for $R_N$ the convergence path of Ada Newton becomes quadratic eventually when the size of the training set becomes equal to the size of the full dataset. It follows that the advantage of Ada Newton with respect to SAGA is more significant if we look for a sub-optimality less than $V_n$. We have observed similar performances for other datasets such as A9A, COVTYPE, and SUSY.

\section{Discussions}

As explained in Section \ref{sec_convergence}, Theorem \ref{the_main_result_theorem} holds because condition \eqref{cond_1} makes $\bbw_m$ part of the quadratic convergence region of $R_n$. From this fact, it follows that the Newton iteration makes the suboptimality gap $R_{n}(\bbw_n)- R_{n}(\bbw_n^*)$ the square of the suboptimality gap $R_{n}(\bbw_m)- R_{n}(\bbw_n^*)$. This yields condition \eqref{cond_2} and is the fact that makes Newton steps valuable in increasing the sample size. If we replace Newton iterations by any method with linear convergence rate, the orders of both sides on condition \eqref{cond_2} are the same. This would make aggressive increase of the sample size unlikely.

In Section \ref{sec_intro} we pointed out four reasons that challenge the development of stochastic Newton methods. It would not be entirely accurate to call Ada Newton a stochastic method because it doesn't rely on stochastic descent directions. It is, nonetheless, a method for ERM that makes pithy use of the dataset. The challenges listed in Section \ref{sec_intro} are overcome by Ada Newton because:

\begin{itemize}
\item[\bf(i)] Ada Newton does not use line searches. Optimality improvement is guaranteed by increasing the sample size.

\item[\bf(ii)] The advantages of Newton's method are exploited by increasing the sample size at a rate that keeps the solution for sample size $m$ in the quadratic convergence region of the risk associated with sample size $n = \alpha m$. This allows aggressive growth of the sample size.

\item[\bf(iii)] The ERM problem is not necessarily strongly convex. A regularization of order $V_n$ is added to construct the empirical risk $R_n$

\item[\bf(iv)] Ada Newton inverts only $\log_\alpha N$ Hessians. 
\end{itemize}

It is fair to point out that items (ii) and (iv) are true only to the extent that the damped phase in Algorithm \ref{alg:AdaNewton} is not significant. Our numerical experiments indicate that this is true but the conclusion is not warranted by out theoretical bounds except when the dataset is very large. This suggests the bounds are loose and that further research is warranted to develop tighter bounds.

\section*{Acknowledgement}

We thank Hadi Daneshmand, Aurelien Lucci, and Thomas Hofmann for useful discussions on the use of adaptive sample sizes for solving large-scale ERM problems and on the importance of using adaptive regularization coefficients.

{\small{
\bibliography{bibliography,}

\begin{thebibliography}{10}

\bibitem{bartlett2006convexity}
Peter~L Bartlett, Michael~I Jordan, and Jon~D McAuliffe.
\newblock Convexity, classification, and risk bounds.
\newblock {\em Journal of the American Statistical Association},
  101(473):138--156, 2006.

\bibitem{beck2009fast}
Amir Beck and Marc Teboulle.
\newblock A fast iterative shrinkage-thresholding algorithm for linear inverse
  problems.
\newblock {\em SIAM journal on imaging sciences}, 2(1):183--202, 2009.

\bibitem{bordes2009sgd}
Antoine Bordes, L{\'e}on Bottou, and Patrick Gallinari.
\newblock Sgd-qn: Careful quasi-newton stochastic gradient descent.
\newblock {\em The Journal of Machine Learning Research}, 10:1737--1754, 2009.

\bibitem{bousquet2008tradeoffs}
Olivier Bousquet and L{\'e}on Bottou.
\newblock The tradeoffs of large scale learning.
\newblock In {\em Advances in Neural Information Processing Systems}, pages
  161--168, 2008.

\bibitem{boyd04}
Stephen Boyd and Lieven Vandenberghe.
\newblock {\em Convex Optimization}.
\newblock Cambridge University Press, New York, NY, USA, 2004.

\bibitem{daneshmand2016starting}
Hadi Daneshmand, Aurelien Lucchi, and Thomas Hofmann.
\newblock Starting small--learning with adaptive sample sizes.
\newblock {\em arXiv preprint arXiv:1603.02839}, 2016.

\bibitem{defazio2014saga}
Aaron Defazio, Francis Bach, and Simon Lacoste-Julien.
\newblock Saga: A fast incremental gradient method with support for
  non-strongly convex composite objectives.
\newblock In {\em Advances in Neural Information Processing Systems}, pages
  1646--1654, 2014.

\bibitem{erdogdu2015convergence}
Murat~A Erdogdu and Andrea Montanari.
\newblock Convergence rates of sub-sampled newton methods.
\newblock In {\em Advances in Neural Information Processing Systems}, pages
  3034--3042, 2015.

\bibitem{frostig2014competing}
Roy Frostig, Rong Ge, Sham~M Kakade, and Aaron Sidford.
\newblock Competing with the empirical risk minimizer in a single pass.
\newblock {\em arXiv preprint arXiv:1412.6606}, 2014.

\bibitem{gower2016stochastic}
Robert~M Gower, Donald Goldfarb, and Peter Richt{\'a}rik.
\newblock Stochastic block bfgs: Squeezing more curvature out of data.
\newblock {\em arXiv preprint arXiv:1603.09649}, 2016.

\bibitem{gurbuzbalaban2015globally}
Mert G{\"u}rb{\"u}zbalaban, Asuman Ozdaglar, and Pablo Parrilo.
\newblock A globally convergent incremental newton method.
\newblock {\em Mathematical Programming}, 151(1):283--313, 2015.

\bibitem{johnson2013accelerating}
Rie Johnson and Tong Zhang.
\newblock Accelerating stochastic gradient descent using predictive variance
  reduction.
\newblock In {\em Advances in Neural Information Processing Systems}, pages
  315--323, 2013.

\bibitem{konevcny2013semi}
Jakub Kone{\v{c}}n{\`y} and Peter Richt{\'a}rik.
\newblock Semi-stochastic gradient descent methods.
\newblock {\em arXiv preprint arXiv:1312.1666}, 2013.

\bibitem{mokhtari2014res}
Aryan Mokhtari and Alejandro Ribeiro.
\newblock Res: Regularized stochastic bfgs algorithm.
\newblock {\em Signal Processing, IEEE Transactions on}, 62(23):6089--6104,
  2014.

\bibitem{JMLR:v16:mokhtari15a}
Aryan Mokhtari and Alejandro Ribeiro.
\newblock Global convergence of online limited memory bfgs.
\newblock {\em Journal of Machine Learning Research}, 16:3151--3181, 2015.

\bibitem{moritz2015linearly}
Philipp Moritz, Robert Nishihara, and Michael~I Jordan.
\newblock A linearly-convergent stochastic l-bfgs algorithm.
\newblock {\em arXiv preprint arXiv:1508.02087}, 2015.

\bibitem{nesterov1998introductory}
Yu~Nesterov.
\newblock Introductory lectures on convex programming volume i: Basic course.
\newblock 1998.

\bibitem{nesterov2007gradient}
Yurii Nesterov et~al.
\newblock Gradient methods for minimizing composite objective function.
\newblock Technical report, UCL, 2007.

\bibitem{polyak1992acceleration}
Boris~T Polyak and Anatoli~B Juditsky.
\newblock Acceleration of stochastic approximation by averaging.
\newblock {\em SIAM Journal on Control and Optimization}, 30(4):838--855, 1992.

\bibitem{robbins1951stochastic}
Herbert Robbins and Sutton Monro.
\newblock A stochastic approximation method.
\newblock {\em The Annals of Mathematical Statistics}, pages 400--407, 1951.

\bibitem{roux2012stochastic}
Nicolas~L Roux, Mark Schmidt, and Francis~R Bach.
\newblock A stochastic gradient method with an exponential convergence rate for
  finite training sets.
\newblock In {\em Advances in Neural Information Processing Systems}, pages
  2663--2671, 2012.

\bibitem{schraudolph2007stochastic}
Nicol~N Schraudolph, Jin Yu, and Simon G{\"u}nter.
\newblock A stochastic quasi-newton method for online convex optimization.
\newblock In {\em International Conference on Artificial Intelligence and
  Statistics}, pages 436--443, 2007.

\bibitem{shalev2010learnability}
Shai Shalev-Shwartz, Ohad Shamir, Nathan Srebro, and Karthik Sridharan.
\newblock Learnability, stability and uniform convergence.
\newblock {\em The Journal of Machine Learning Research}, 11:2635--2670, 2010.

\bibitem{shalev2013stochastic}
Shai Shalev-Shwartz and Tong Zhang.
\newblock Stochastic dual coordinate ascent methods for regularized loss.
\newblock {\em The Journal of Machine Learning Research}, 14:567--599, 2013.

\bibitem{shalev2016accelerated}
Shai Shalev-Shwartz and Tong Zhang.
\newblock Accelerated proximal stochastic dual coordinate ascent for
  regularized loss minimization.
\newblock {\em Mathematical Programming}, 155(1-2):105--145, 2016.

\bibitem{vapnik1998statistical}
Vladimir Vapnik.
\newblock {\em The nature of statistical learning theory}.
\newblock Springer Science \& Business Media, 2013.

\bibitem{xiao2014proximal}
Lin Xiao and Tong Zhang.
\newblock A proximal stochastic gradient method with progressive variance
  reduction.
\newblock {\em SIAM Journal on Optimization}, 24(4):2057--2075, 2014.

\bibitem{zhang2013linear}
Lijun Zhang, Mehrdad Mahdavi, and Rong Jin.
\newblock Linear convergence with condition number independent access of full
  gradients.
\newblock In {\em Advances in Neural Information Processing Systems}, pages
  980--988, 2013.

\end{thebibliography}
\bibliographystyle{plain}
}}


\section{Appendix}

In this section we study the proofs of Propositions \ref{main_prop} and  \ref{main_theorem_444}. To do so, first we prove Lemmata \ref{lemma_1} and \ref{lemma_2} which are intermediate results that we use in proving the mentioned propositions. 

We start the analysis by providing an upper bound for the difference between the loss functions $L_n$ and $L_m$. The upper bound is studied in the following lemma which uses the condition in \eqref{eqn_loss_minus_erm}.

\begin{lemma}\label{lemma_1}
Consider $L_n$ and $L_m$ as the empirical losses of the sets $\S_n$ and $\S_m$, respectively, where they are chosen such that $\S_m\subset \S_n$. If we define $n$ and $m$ as the number of samples in the training sets $\S_n$ and $\S_m$, respectively, then the absolute value of the difference between the empirical losses is bounded above by
\begin{align}\label{proof_eq_44}
|L_n(\bbw)-L_m(\bbw)| \leq \frac{n-m}{n} \left(V_{n-m}+V_m\right),\qquad \whp.
\end{align}
for any $\bbw$.
\end{lemma}

	\begin{proof}
First we characterize the difference between the difference of the loss functions associated with the sets $\S_m$ and $\S_n$. To do so, consider the difference 
\begin{equation}\label{proof_eq_11}
L_n(\bbw)-L_m(\bbw)= \frac{1}{n} \sum_{i\in\S_n} f_i(\bbw)-\frac{1}{m} \sum_{i\in\S_m} f_i(\bbw).
\end{equation}
Notice that the set $\S_m$ is a subset of the set $\S_n$ and we can write $\S_n=\S_m\cup\S_{n-m}$. Thus, we can rewrite the right hand side of \eqref{proof_eq_11} as 
\begin{align}\label{proof_eq_22}
L_n(\bbw)-L_m(\bbw)
&= \frac{1}{n} \left[\sum_{i\in\S_m} f_i(\bbw)+\sum_{i\in\S_{n-m}} f_i(\bbw)\right]-\frac{1}{m} \sum_{i\in\S_m} f_i(\bbw)\nonumber\\
& =  \frac{1}{n} \sum_{i\in\S_{n-m}} f_i(\bbw) - \frac{n-m}{mn}\sum_{i\in\S_m} f_i(\bbw).
\end{align}
Factoring $(n-m)/n$ from the terms in the right hand side of \eqref{proof_eq_22} follows
\begin{align}\label{proof_eq_33}
L_n(\bbw)-L_m(\bbw)= \frac{n-m}{n}\left[ \frac{1}{n-m} \sum_{i\in\S_{n-m}} f_i(\bbw) - \frac{1}{m}\sum_{i\in\S_m} f_i(\bbw)\right]
\end{align}
Now add and subtract the statistical loss $L(\bbw)$ to obtain
\begin{align}\label{proof_eq_44}
|L_n(\bbw)-L_m(\bbw)| 
&= \frac{n-m}{n}\left|\frac{1}{n-m} \sum_{i\in\S_{n-m}} f_i(\bbw) -L(\bbw)+L(\bbw)- \frac{1}{m}\sum_{i\in\S_m} f_i(\bbw)\right|\nonumber\\
&\leq \frac{n-m}{n} \left(V_{n-m}+V_m\right).
\end{align}
where the last inequality follows by using the triangle inequality and the upper bound in \eqref{eqn_loss_minus_erm}. 
\end{proof}

The result in Lemma \ref{lemma_1} shows that the upper bound for the difference between the loss functions associated with the sets $\S_m$ and $\S_n$ where $\S_m\subset \S_n$ is proportional to the difference between the size of these two sets $n-m$. This result will help us later to understand how much we can increase the size of the training set at each iteration. In other words, how large the difference $n-m$ could be, while we have the statistical accuracy. 

In the following lemma, we characterize an upper bound for the norm of the optimal argument $\bbw_n^*$ of the empirical risk $R_n(\bbw)$ in terms of the norm of statistical average loss $L(\bbw)$ optimal argument $\bbw^*$.

\begin{lemma}\label{lemma_2}
Consider $L_n$ as the empirical loss of the set $\S_n$ and $L$ as the statistical average loss. Moreover, recall $\bbw^*$ as the optimal argument of the statistical average loss $L$, i.e., $\bbw^*=\argmin_{\bbw} L(\bbw)$. If Assumption \ref{convexity_lip_assumption} holds, then the norm of the optimal argument $\bbw_n^*$ of the regularized empirical risk $R_n(\bbw):=L_n(\bbw)+cV_n\|\bbw\|^2$ is bounded above by 
\begin{align}\label{claim_lemma_2}
\|\bbw_n^*\|^2 \leq \frac{4}{c}+\|\bbw^*\|^2,  \qquad \whp.
\end{align}
\end{lemma}

	\begin{proof}
The optimality condition of $\bbw_n^*$ for the the regularized empirical risk $R_n(\bbw)=L_n(\bbw)+ (cV_n)/2\|\bbw\|^2$ implies that 
\begin{equation}\label{proof_lemma_2_eq_11}
L_n(\bbw_n^*)+ \frac{cV_n}{2}\|\bbw_n^*\|^2
\leq   L_n(\bbw^*)+ \frac{cV_n}{2}\|\bbw^*\|^2.
\end{equation}
By regrouping the terms we obtain that the squared norm $\|\bbw_n^*\|^2$ is bonded above by
\begin{equation}\label{proof_lemma_2_eq_22}
\|\bbw_n^*\|^2
\leq \frac{2}{cV_n}  \left(L_n(\bbw^*)-L_n(\bbw_n^*)\right)+ \|\bbw^*\|^2.
\end{equation}
We proceed to bound the difference $L_n(\bbw^*)-L_n(\bbw_n^*)$. By adding and subtracting the terms $L(\bbw^*)$ and $L(\bbw_n^*)$ we obtain that 
\begin{equation}\label{proof_lemma_2_eq_33}
L_n(\bbw^*)-L_n(\bbw_n^*) = \big[L_n(\bbw^*)-L(\bbw^*)\big]
+\big[L(\bbw^*)-L(\bbw_n^*)\big]
+\big[L(\bbw_n^*)-L_n(\bbw_n^*)\big].
\end{equation}
Notice that the second bracket in \eqref{proof_lemma_2_eq_33} is non-positive  since $L(\bbw^*)\leq L(\bbw_n^*)$. Therefore, it is bounded by $0$. According to \eqref{eqn_loss_minus_erm}, the first and third brackets in \eqref{proof_lemma_2_eq_33} are with high probability bounded above by $V_n$. Replacing these upper bounds by the brackets in \eqref{proof_lemma_2_eq_33} yields 
\begin{equation}\label{proof_lemma_2_eq_44}
L_n(\bbw^*)-L_n(\bbw_n^*) \leq  2V_n.
\end{equation}
Substituting the upper bound in \eqref{proof_lemma_2_eq_44} into \eqref{proof_lemma_2_eq_22} follows the claim in \eqref{claim_lemma_2}.
\end{proof}

\subsection{Proof of Proposition \ref{main_prop}}
From the self-concordance analysis of Newton's method we know that the variable $\bbw_m$ is in the neighborhood that Newton's method has a quadratic convergence rate if $\lambda_n (\bbw_m)\leq 1/4$; see e.g., Chapter 9 of \cite{boyd04}. We proceed to come up with a condition for the quadratic convergence phase which guarantees that $ \lambda_n(\bbw_m)<1/4$ and $\bbw_m$ is in the local neighborhood of the optimal argument of $R_n$. Recall that we have a $\bbw_m$ which has sub-optimality $V_m$ for $R_m$.  We then proceed to enlarge the sample size to $n$ and start from the observation that we can bound $\lambda_n(\bbw_m)$ as
\begin{align}\label{proof_prop_new_10}
   \lambda_n(\bbw_m) =    \|\nabla R_{n}(\bbw_m)\|_{\bbH_n^{-1}} 
                     \leq \|\nabla R_{m}(\bbw_m)\|_{\bbH_n^{-1}}  
                            + \|\nabla R_{n}(\bbw_m) - \nabla R_{m}(\bbw_m) \|_{\bbH_n^{-1}} ,
\end{align}
where we have used the definition $\bbH_n= \nabla^2 R_{n}(\bbw_m) $. Note that the weighted norm $\|\bba\|_\bbA$ for vector $\bba$ and matrix $\bbA$ is equal to $\|\bba\|_\bbA=(\bba^T\bbA\bba)^{1/2}$. First, we bound the norm $   \|\nabla R_{n}(\bbw_m)\|_{\bbH_n^{-1}} $ in \eqref{proof_prop_new_10}. Notice that the Hessian $\nabla^2 R_{n}(\bbw_m)$ can be written as $\nabla^2 L_{n}(\bbw_m) +cV_n\bbI$. Thus, the eigenvalues of the Hessian $\bbH_n=\nabla^2 R_{n}(\bbw_m) $ are bounded below by $cV_n$ and consequently the eigenvalues of the Hessian inverse $\bbH_n^{-1}=\nabla^2 R_{n}(\bbw_m) ^{-1}$ are upper bounded by $1/(cV_n)$. This bound implies that $\|\bbH_n^{-1}\|\leq 1/(cV_n)$. Moreover, from Theorem 2.1.5 of \cite{nesterov1998introductory}, we know that the Lipschitz continuity of the gradients $\nabla R_{m}(\bbw)$ with constant $M+cV_m$ implies that 
\begin{equation}\label{proof_prop_new_20}
{\|\nabla R_{m}(\bbw_m)\|^2\leq 2 (M+cV_m) ( R_{m}(\bbw_m) - R_{m}(\bbw_m^*))}
\leq 2 (M+cV_m) V_m,
\end{equation}
where the last inequality holds comes from the condition that $R_{m}(\bbw_m) - R_{m}(\bbw_m^*)\leq V_m$.
Considering the upper bound for $\|\nabla R_{m}(\bbw_m)\|^2$ in \eqref{proof_prop_new_20} and the inequality $\|\nabla^2 R_{n}(\bbw_m) ^{-1}\|\leq 1/(cV_n)$ we can write 
\begin{align}\label{proof_prop_new_30}
  \|\nabla R_{m}(\bbw_m)\|_{\bbH_n^{-1}} 
      =    \Big[\nabla R_{m}(\bbw_m)^T\bbH_n^{-1}\nabla R_{m}(\bbw_m)\Big]^{1/2}
      \leq \left(\frac{2(M+cV_m)V_{m}}{cV_n}\right)^{1/2}.
\end{align}

Now we proceed to bound the second the term in \eqref{proof_prop_new_10}. The definition of the risk function the gradient can be written as $\nabla R_n(\bbw)=\nabla L_n(\bbw)+(cV_n)\bbw$. Thus, we can derive an upper bound for the difference $ \| \nabla R_{n}(\bbw_m) - \nabla R_{m}(\bbw_m)\|$ as
\begin{align}\label{proof_prop_new_40}
  & \| \nabla R_{n}(\bbw_m) - \nabla R_{m}(\bbw_m)\|\nonumber\\
       &\qquad \leq    \|\nabla L_n(\bbw_m) - \nabla L_m(\bbw_m)\|
                + c(V_m-V_n)\|\bbw_m\|  \nonumber \\          
       &\qquad \leq   \|\nabla L_n(\bbw_m) - \nabla L_m(\bbw_m)\|
                + c(V_m-V_n)\|\bbw_m-\bbw_m^*\|+ c(V_m-V_n)\|\bbw_m^*\|,
\end{align}
where in the second inequality we have used the triangle inequality and replaced $\|\bbw_m\| $ by its upper bound $\|\bbw_m-\bbw_m^*\|+\|\bbw_m^*\|$. By following the steps in \eqref{proof_eq_11}-\eqref{proof_eq_44} we can show that the difference $\|\nabla L_n(\bbw_m) - \nabla L_m(\bbw_m)\|$ is bounded above by 
\begin{align}\label{proof_prop_new_50}
\|\nabla L_n(\bbw)-\nabla L_m(\bbw)\| 
&\leq \frac{n-m}{n}\left\| \nabla L_{n-m}(\bbw)-\nabla L(\bbw) \right\| + \frac{n-m}{n} \left\|\nabla L_m(\bbw) -\nabla L(\bbw)  \right\|\nonumber\\
&\leq \frac{2(n-m)}{n} V_n^{1/2},
\end{align}
where the second inequality uses the condition that $\left\|\nabla L_m(\bbw) -\nabla L(\bbw)  \right\|\leq V_m^{1/2}$. Note that the strong convexity of the risk $R_{m}$ with parameter $cV_m$ yields 
\begin{align}\label{proof_prop_new_60}
\|\bbw_m-\bbw_m^*\|^2 \leq \frac{2}{cV_m} (R_{m}(\bbw_m) - R_{m}(\bbw_m^*))
\leq \frac{2}{c}.
\end{align}
Thus, by considering the inequalities in \eqref{proof_prop_new_50} and \eqref{proof_prop_new_60} we can show that upper bound in \eqref{proof_prop_new_40} can be replaced by 
\begin{align}\label{proof_prop_new_70}
   \| \nabla R_{n}(\bbw_m) - \nabla R_{m}(\bbw_m)\|         
       \leq \frac{2(n-m)}{n} V_n^{1/2} +(\sqrt{2c}+c\|\bbw_m^*\|)(V_m-V_n) .
\end{align}
Substituting the upper bounds in \eqref{proof_prop_new_30} and \eqref{proof_prop_new_70} for the first and second summands in \eqref{proof_prop_new_10}, respectively, follows the inequality 
\begin{align}\label{proof_prop_new_80}
   \lambda_n(\bbw_m) \leq 
   \left(\frac{2(M+cV_m)V_{m}}{cV_n}\right)^{1/2}
                            + \frac{{(2(n-m)}/{n}) V_n^{1/2} +(\sqrt{2c}+c\|\bbw_m^*\|)(V_m-V_n)}{(cV_n)^{1/2}}.
\end{align}
Note that the result in \eqref{claim_lemma_2} shows that $\|\bbw_m^*\|^2 \leq ({4}/{c})+\|\bbw^*\|^2$ with high probability. This observation follows that $\|\bbw_m^*\|$ is bounded above by $(2/\sqrt{c})+\|\bbw^*\|$. Replacing the norm $\|\bbw_m^*\|$ in \eqref{proof_prop_new_80} by the upper bound $ ({2}/{\sqrt{c}})+\|\bbw^*\|$ follows
\begin{align}\label{proof_prop_new_90}
   \lambda_n(\bbw_m) \leq 
   \left(\frac{2(M+cV_m)V_{m}}{cV_n}\right)^{1/2}
                            + \frac{{(2(n-m)}/{n}) V_n^{1/2} +(\sqrt{2c}+2\sqrt{c}+c\|\bbw^*\|)(V_m-V_n)}{(cV_n)^{1/2}}.
\end{align}
As we mentioned previously, the variable $\bbw_m$ is in the neighborhood that Newton's method has a quadratic convergence rate for the function $R_n$ if the condition $\lambda_n (\bbw_m)\leq 1/4$ holds. Hence, if the right hand side of \eqref{proof_prop_new_90} is bounded above by $1/4$ we can conclude that $\bbw_m$ is in the local neighborhood and the proof is complete.

\subsection{Proof of Proposition \ref{main_theorem_444}}
Note that the difference $ R_{n}(\bbw_m) - R_{n}(\bbw_n^*)$ can be written as
\begin{align}\label{proof_eq_2_11}
 R_{n}(\bbw_m) - R_n(\bbw_n^*)
 &=  R_{n}(\bbw_m) -  R_{m}(\bbw_m) +  R_{m}(\bbw_m) - R_{m}(\bbw_m^*) \nonumber\\
 &\qquad+ R_{m}(\bbw_m^*) -R_{m}(\bbw_n^*) +R_{m}(\bbw_n^*) - R_{n}(\bbw_n^*).
\end{align}
We proceed to bound the differences in \eqref{proof_eq_2_11}. To do so, note that the difference $R_{n}(\bbw_m) -  R_{m}(\bbw_m)$ can be simplified as 
\begin{align}\label{proof_eq_2_22}
R_{n}(\bbw_m) -  R_{m}(\bbw_m)
&=L_n(\bbw_m)-L_m(\bbw_m)+\frac{c{(V_{n}-V_m)}}{2}\|\bbw_m\|^2
\nonumber\\
&\leq  L_n(\bbw)-L_m(\bbw),
\end{align}
where the inequality follows from the fact that $V_{n}<V_m$ and $V_{n}-V_m$ is negative.
It follows from the result in Lemma \ref{lemma_1} that the right hand side of \eqref{proof_eq_2_22} is bounded by $ ({n-m})/{n} \left(V_{n-m}+V_m\right)$. Therefore,
\begin{equation}\label{proof_eq_2_44}
R_{n}(\bbw_m) -  R_{m}(\bbw_m)  \leq   \frac{n-m}{n} \left(V_{n-m}+V_m\right).
\end{equation}
According to the fact that $\bbw_m$ as an $V_m$ optimal solution for the sub-optimality $ R_{m}(\bbw_m) - R_{m}(\bbw_m^*) $ we know that 
\begin{equation}\label{proof_eq_2_55}
 R_{m}(\bbw_m) - R_{m}(\bbw_m^*)   \leq   V_m.
\end{equation}
Based on the definition of $\bbw_{m}^*$ which is the optimal solution of the risk $R_{m}$, the third difference in \eqref{proof_eq_2_11} which is $ R_{m}(\bbw_m^*) -R_{m}(\bbw_n^*) $ is always negative. I.e., 
\begin{equation}\label{proof_eq_2_66}
R_{m}(\bbw_m^*) -R_{m}(\bbw_n^*)    \leq  0.
\end{equation}
Moreover, we can use the triangle inequality to bound the difference $ R_{m}(\bbw_n^*) - R_{n}(\bbw_n^*)$ in \eqref{proof_eq_2_11} as
\begin{align}\label{proof_eq_2_77}
R_{m}(\bbw_n^*) - R_{n}(\bbw_n^*) 
& =  L_m(\bbw_n^*)-L_n(\bbw_n^*) + \frac{c(V_m-V_n)}{2} \|\bbw_n^*\|^2\nonumber\\
& \leq  \frac{n-m}{n} \left(V_{n-m}+V_m\right)+ \frac{c(V_m-V_n)}{2}\|\bbw_n^*\|^2.
\end{align}
Replacing the differences in \eqref{proof_eq_2_11} by the upper bounds in \eqref{proof_eq_2_44}-\eqref{proof_eq_2_77} follows 
\begin{equation}\label{proof_eq_2_88}
 R_{n}(\bbw_m) - R_n(\bbw_n^*) \leq V_m +  \frac{2(n-m)}{n} \left(V_{n-m}+V_m\right)+ \frac{c(V_m-V_n)}{2}\|\bbw_n^*\|^2
\quad \text{w.h.p.}
\end{equation}
Substitute $\|\bbw_n^*\|^2$ in \eqref{proof_eq_2_88} by the upper bound in \eqref{claim_lemma_2} to obtain the result in \eqref{theorem_result_444}.

%
%
%
%
%
%

\end{document}